\documentclass[sigconf]{acmart}

\AtBeginDocument{%
  \providecommand\BibTeX{{%
    \normalfont B\kern-0.5em{\scshape i\kern-0.25em b}\kern-0.8em\TeX}}}

\begin{document}

\title{Tackling the Imbalance for GNNs}

\author{Rui Wang}
\email{rwang\_ruiwang@tju.edu.cn}
\affiliation{%
  \institution{National Center for Applied Mathematics, Tianjin University}
  \city{Tianjin}
  \country{China}
}

\author{Weixuan Xiong}
\email{weixuan@tju.edu.cn}
\affiliation{%
  \institution{National Center for Applied Mathematics, Tianjin University}
  \city{Tianjin}
  \country{China}
}

\author{Qinghu Hou}
\email{qh\_hou@tju.edu.cn}
\affiliation{%
  \institution{National Center for Applied Mathematics, Tianjin University}
  \city{Tianjin}
  \country{China}
}

\author{Ou Wu}
\authornote{Corresponding author.}
\email{wuou@tju.edu.cn}
\affiliation{%
  \institution{National Center for Applied Mathematics, Tianjin University}
  \city{Tianjin}
  \country{China}
}






\begin{abstract} 
Different from deep neural networks for non-graph data classification, graph neural networks (GNNs) leverage the information exchange between nodes (or samples) when representing nodes. The category distribution shows an imbalance or even a highly-skewed trend on nearly all existing benchmark GNN data sets. The imbalanced distribution will cause misclassification of nodes in the minority classes, and even cause the classification performance on the entire data set to decrease. This study explores the effects of the imbalance problem on the performances of GNNs and proposes new methodologies to solve it. First, a node-level index, namely, the label difference index ($LDI$), is defined to quantitatively analyze the relationship between imbalance and misclassification. The less samples in a class, the higher the value of its average $LDI$; the higher the $LDI$ of a sample, the more likely the sample will be misclassified. We define a new loss and propose four new methods based on $LDI$. Experimental results indicate that the classification accuracies of the three among our proposed four new methods are better in both transductive and inductive settings. The $LDI$ can be applied to other GNNs.
\end{abstract}


\pagestyle{empty}

\settopmatter{printacmref=false}
\renewcommand\footnotetextcopyrightpermission[1]{}


\maketitle

\section{Introduction}
As an increasing number of applications involve graph data, researchers have designed graph neural networks (GNNs) to process graph data. GNNs generally use the connections between nodes to exchange information in neighborhoods to obtain a better representation for each node. Therefore, a unique problem faced by GNNs is misclassification caused by neighbor nodes. Nodes in minority classes are more likely to be neighbored by those in majority classes, resulting in low accuracy.

Real-world data usually conforms to an imbalanced even a highly-skewed distribution, in which majority classes occupy most of the data proportion. Deep neural networks also suffer when the training data are highly imbalanced~\cite{Cui_2019_CVPR_1, buda2018systematic, ren18l2rw_4}. 



This study investigates the imbalance problem for GNNs, and reveals more details for the relationship between imbalance and misclassification on graph data. Furthermore, based on the state-of-the-art algorithms for dealing with the imbalance of non-graph data, new methodologies are proposed. First, we define the label difference index ($LDI$) to measure the likelihood of a node being misclassified based on the category distribution of its neighborhood. The larger the index, the more likely the node can be misclassified. The $LDI$ of a node is determined by the category distribution of its neighbor nodes. Nodes neighbored by heterophily nodes\footnote{Heterophily nodes are the samples (or nodes) that belong to different categories \cite{zhu2020graph}. } are more likely to be with larger $LDI$s, and easy to be negatively affected, resulting in misclassification. In addition, $LDI$ is also affected by the global category distribution of the entire data set, and statistics show that the average $LDI$s of the nodes in minority classes are usually large. Thus, the relationship between the imbalance and misclassification is established. Second, we define a new loss and propose four new methods, namely, improved focal loss (iFL), Graph Re-sampling (GRS), Graph Re-weighting (GRW), Graph Metric Learning (GML), and Graph Bilateral-branch Network (GBBN). Experiments on several graph benchmark data sets show that except for GML, the new methods are better than the original GNNs. Imbalanced sampling with $LDI$ can help to further improve performance.

Our contributions are summarized as follows:
\begin{itemize}
\item[(1)] A node-level index (i.e., $LDI$) is defined to characterize the neighborhood of a node (or a sample) in a graph. On the basis of $LDI$, the relationship between imbalance and misclassification is analyzed and useful observations are obtained.
\item[(2)] A new loss and four new methods are proposed based on the index $LDI$. Extensive experiments on benchmark data sets indicate that the proposed methodology achieves better results than existing  imbalance learning methods on graphs. Further, several classical GNNs are improved by utilizing $LDI$.
\end{itemize}

\section{Related work}
Classical GNNs include GCN~\cite{kipf2017semi_11}, SGC \cite{pmlr-v97-wu19e_12}, GAT~\cite{velickovic2018graph_13}, and so on. Kipf et al.~\cite{kipf2017semi_11} expanded the traditional convolution neural network (CNN) on high-dimensional graph data to obtain GCN. GCN iteratively updated each node's representation through the message exchange with their neighbor nodes. Wu et al.~\cite{pmlr-v97-wu19e_12} reduced the complexity of GCN by repeatedly eliminating the nonlinearity between GCN layers and folding the resulting function into a linear transformation to obtain SGC. Veli{\v{c}}kovi{\'{c}} et al.~\cite{velickovic2018graph_13} proposed GAT based on the attention mechanism to classify graph data.

In the research of graph, researchers usually define indices (e.g., degree and centrality) to characterize the properties of a graph. Likewise, a number of indices are also defined to measure the properties of the involved graphs in the research of GNNs. For example, graph homophily was proposed by Pei et al. \cite{pei2020geom} to characterize the degree that similar nodes connect together. Graph heterophily was proposed by Zhu et al. \cite{zhu2020beyond} to measure the graph homophily level. Two graph smoothness metrics, namely, feature smoothness and label smoothness, were proposed by Hou et al. \cite{Hou2020Measuring_2} to help understand the use of graph information in GNNs. Almost all existing graph indices in GNNs are graph-level or category-level\footnote{Several indices first define a node-level metrics and using the average of all nodes as the graph-level or category-level index. The node-level index is then no longer used.}. These indices are only used to analyze graph characteristics and the learning performances. They are not involved in the training process. To our knowledge, there is no node-level index. This study will define a node-level index which can be directly used in both characteristics analyses and model training.

Machine learning has been increasingly applied in recent years. The imbalance has become a hot research topic. Kang et al.~\cite{kang2020decoupling_19} divided the learning of a classification model into two steps. Zhou et al. \cite{Zhou_2020_CVPR_9} proposed a new learning model with two branches, both of which involve re-sampling. Zong et al. \cite{zong2020gnn} proposed GNN-XML to overcome the imbalance problem in extreme multi-label text classification. Using re-weighting, Cui et al.~\cite{Cui_2019_CVPR_1} used a better weighting design to achieve a better imbalance classifier. Liu et al.~\cite{Liu_2020_CVPR_20} created some virtual samples around the minority classes samples to increase the number of minority classes samples. Liu et al.~\cite{Liu_2019_CVPR_21} transferred the visual knowledge of the majority to the minority classes by learning a set of dynamic element vectors. Lin et al. \cite{lin2017focal} proposed a new loss function, focal loss, to deal with the imbalance problem in object detection. 

As to the imbalance of graph data, Min et al. \cite{shi2020multi} conducted a pilot study for this problem. They proposed DR-GCN, adopting a conditional adversarial training together with distribution alignment to learn node representations. They only focus on the transductive setting. Therefore, methods are compared on the transductive setting in Section \ref{results}.

\section{Effect of imbalance for GNNs}
\subsection{Qualitative analysis. }
We first analyze the interaction between the majority and the minority classes in an imbalanced data set of a non-graph classification task from the perspective of training loss. The loss can be expressed as follows:
\begin{equation}
\mathcal{L}_{total}=\mathcal{L}_{majority}+\mathcal{L}_{minority},
\end{equation}
where $\mathcal{L}_{majority}$ is the loss of the majority classes and $\mathcal{L}_{minority}$ is the loss of the minority classes. Because the number of samples in the majority classes (majority samples for brevity) is much greater than that of samples in the minority classes (minority samples for brevity), the loss is composed mainly of majority samples, and the gradients of model optimization are determined more by majority samples. As a result, both the feature representation and the decision layers of the training model are optimized towards a high classification accuracy of the majority classes in each training epoch, which may lead to the overfitting of the majority classes and ignores the minority classes while training. Ultimately, the performance of the overall model is damaged.

In graph classification tasks, because the loss function can still be expressed in the above form, the training for GNNs will also be affected by the imbalance, that is, the majority samples determine the gradients in training. However, because the majority and the minority nodes in the graph usually have direct edge links, in addition to the negative effects of the loss of coupling on the minority classes, the nodes in the majority classes may affect the feature representation of the nodes in the minority classes directly through edges. When node heterophily holds, this kind of effect between neighbor nodes is very likely to be negative. 

Information exchange between nodes of the same category is beneficial, while information exchange between nodes of different categories is likely to be harmful in GNNs. Because there are fewer minority nodes, the proportion of nodes in minority classes with heterophily nodes in their neighborhoods may be relatively larger, which may have a greater negative effect. In other words, minority nodes are more likely to be misclassified.

To sum up, imbalance in non-graph data affects both the decision and the feature representation layers through the loss minimization; imbalance in graph data affects both the decision and the feature representation layers according to both the loss minimization and the direct information exchange between nodes. Therefore, imbalance in graph data can have a more serious impact on GNNs training.

\subsection{Quantitative analysis. }
Suppose a graph $G=(V,E)$ and let $x_i \in V$. The category distribution of the $K$-neighbor nodes of $x_i$ is $P_{\mathcal{N}_i}=\{p_{i,1}, \cdots, p_{i,c}, \cdots, p_{i,C}\}$, where $\mathcal{N}_i$ denotes the $K$-neighborhood and $p_{i,c}$ represents the proportion of the $c$th category in $\mathcal{N}_i$. We use 1-neighborhood to reduce the computation complexity in this study. Meanwhile, if $x_i$ belongs to the $c$th category, the distribution of itself is marked as a one-hot vector $P_{\mathcal{I}_i}=\{0, \cdots, 0, 1, 0,  \cdots, 0\}$, where the $c$th element is 1.

We define a node-level index, label difference index ($LDI$), on graph to better characterize the relationship between imbalance and misclassification.

\begin{definition}
\textbf{Label Difference Index (\textit{LDI})}: Given a node $x_i$, the $LDI$ of $x_i$ is
\begin{equation}
LDI_i=\frac{1}{\sqrt{2}}\|P_{\mathcal{N}_i}-P_{\mathcal{I}_i}\|_2.
\end{equation}
\end{definition}

The range of $LDI_i$ is $[0,1]$. The higher the $LDI$ value, the larger the difference between the two distributions. If $x_i$ is an isolated node, then we define its $LDI$ as the average $LDI$ of nodes in the same category. We have the following proposition.

\begin{proposition}\label{lemma1}
Given a node $x_i$ belonging to the $c$th category. When the proportion of homophily nodes in its neighborhood is fixed (i.e., $p_{i,{c}}$ is fixed), the less categories of heterophily nodes, the larger the $LDI$.
\end{proposition}

The proof is shown in the appendix. Figure \ref{fig:fig2} shows an illustrative example. The colors denote the categories. Assume that $P_{\mathcal{I}_5}=\{1,0,0,0\}$. In Figure 1(a), $P_{\mathcal{N}_5}=\{1,0,0,0\}$, then $LDI_5=0$. In Figure 1(b), $P_{\mathcal{N}_5}=\{\frac{3}{4},\frac{1}{4},0,0\}$, then $LDI_5=\frac{1}{4}$. In Figure 1(c), $P_{\mathcal{N}_5}=\{\frac{1}{4},\frac{1}{4},\frac{1}{4},\frac{1}{4}\}$, then $LDI_5=\frac{\sqrt{6}}{4}$. In Figure 1(d), $P_{\mathcal{N}_5}=\{0,\frac{1}{2},\frac{1}{4},\frac{1}{4}\}$, then $LDI_5=\frac{\sqrt{11}}{4}$. In Figure 1(e), $P_{\mathcal{N}_5}=\{0,\frac{1}{2},\frac{1}{2},0\}$, then $LDI_5=\frac{\sqrt{3}}{2}$. In Figure 1(f), $P_{\mathcal{N}_5}=\{0,1,0,0\}$, then $LDI_5=1$. In other words, node 5 in Figure 1(f) is the most easiest to be misclassified. The six $LDI$ values in Figure~\ref{fig:fig2} indicate that our $LDI$ definition is reasonable and Proposition \ref{lemma1} holds. 

\begin{figure}[]
    \centering
    \includegraphics[width=\linewidth]{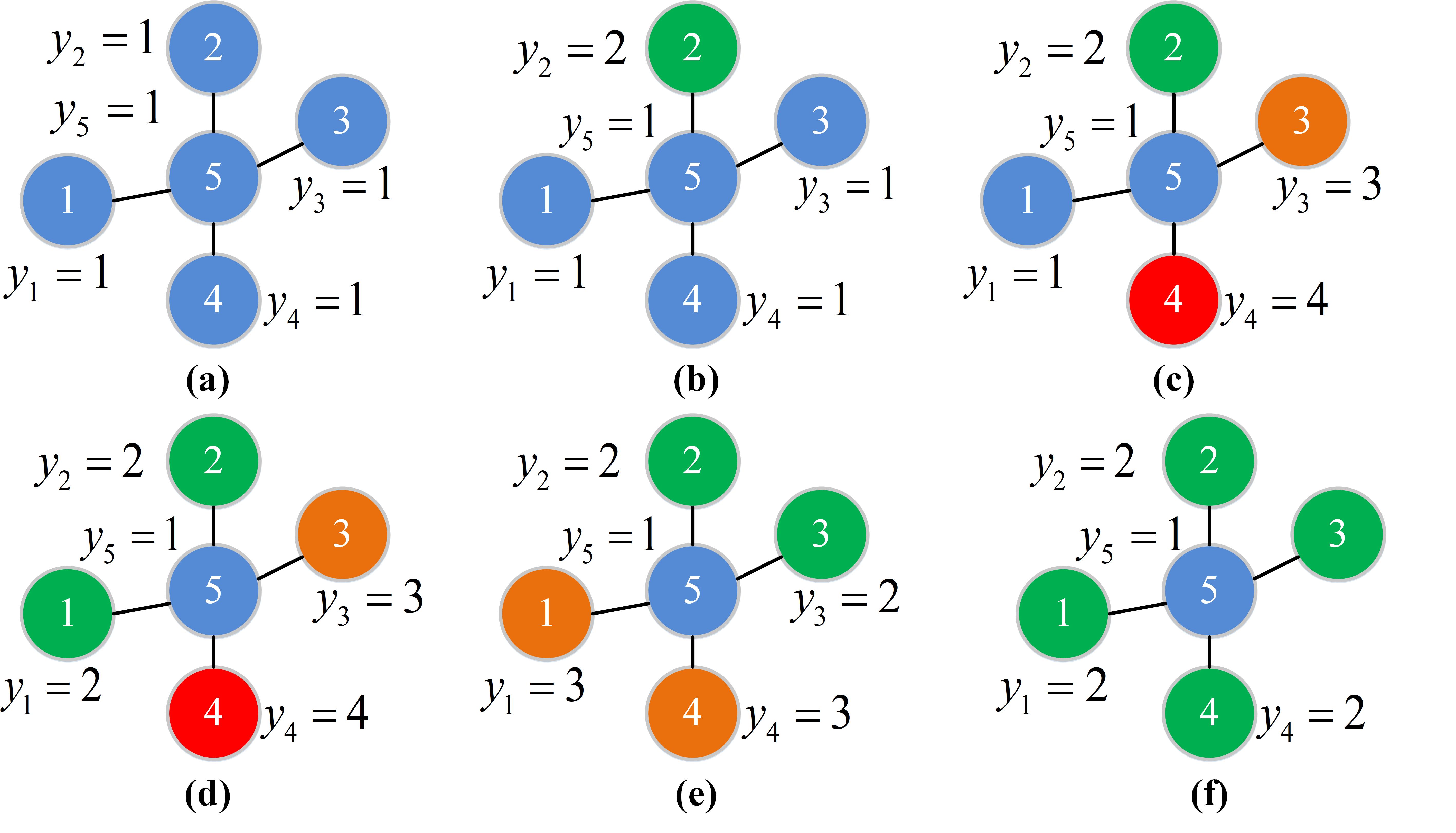}
    \vspace{-0.2in}
    \caption{Examples of nodes with different neighborhoods. The $LDI$ values of node 5 in (a)-(f) are $0$, $\frac{1}{4}$, $\frac{\sqrt{6}}{4}$, $\frac{\sqrt{11}}{4}$, $\frac{\sqrt{3}}{2}$, $1$, respectively.}
    \vspace{-0.25in}
    \label{fig:fig2}
\end{figure}

$LDI$ is a node-level index that can be easily applied to concrete GNN algorithms. Different from the homophily index mentioned in \cite{pei2020geom},  $LDI$ utilizes the distribution of neighbors, and is more fine-grained. Proposition \ref{lemma1} verifies this. The indices of node 5 in Figures \ref{fig:fig2} (d), (e), and (f) are the same (equal to 0) if the homophily index in \cite{pei2020geom} is used. However, $LDI$ can distinguish these three cases better.

Based on $LDI$, we conduct analyses in terms of the differences among categories, the differences among samples within specific categories, the relationship between $LDI$ and accuracies, and the relationship between $LDI$ and misclassification under different numbers of layers ($\#layers$).

\textbf{Average \textit{LDI}s of different categories.} We calculate the average $LDI$s of different categories on five benchmark graph data sets, as shown in Figure~\ref{fig:fig3}. Although the trend of the category distribution (Figure~\ref{fig:fig3} (up)) is not completely the opposite of the trend of the LDI distribution (Figure~\ref{fig:fig3} (middle)) on each data set, the overall trend is that the $LDI$s of minority classes are large, indicating that the proportions of heterophily nodes around the nodes in the minority classes are relatively high. In Figure \ref{fig:fig3} (down), the categories with large average $LDI$s usually have low accuracies.

\begin{figure*}[h]
    \centering
    \includegraphics[width=\linewidth]{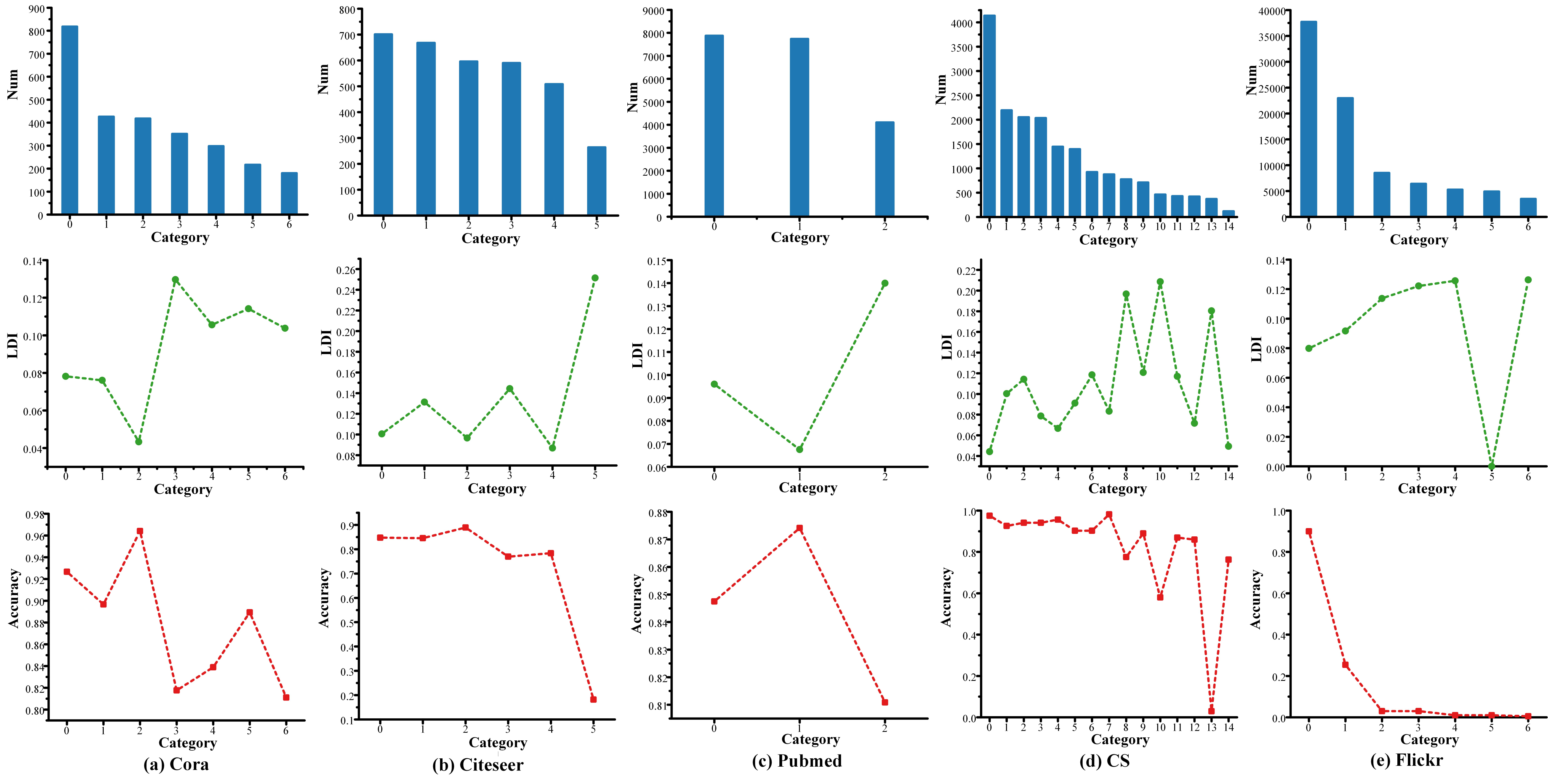}
    \vspace{-0.2in}
    \caption{Category distributions (up), average $LDI$s (middle), and average accuracies (down)  of the five different data sets used in this study. Categories with small average $LDI$s always have high accuracies. On Flickr, Category 5 has a small average $LDI$ but a small average accuracy. This is mainly because loss minimization and feature exchange between nodes negatively affect the accuracy of minority classes, and loss minimization plays a main role.}
    \label{fig:fig3}
\vspace{-0.15in}
\end{figure*}

\textbf{\textit{LDI} distributions within specific categories.} Within a category, the $LDI$s of different samples are also distinct. We arbitrarily select a majority class and a minority class from the graph data set Citeseer~\cite{shi2020multi}. As shown in Figure~\ref{fig:fig4}, samples in the majority class are more concentrated in the small $LDI$ intervals. Contrarily, samples in the minority class not only concentrate in the large $LDI$ intervals. That is, samples with large $LDI$s can be found in the majority classes and samples with small $LDI$s can be found in the minority classes. 

\begin{figure}[]
\centering
\includegraphics[width=1.02\linewidth]{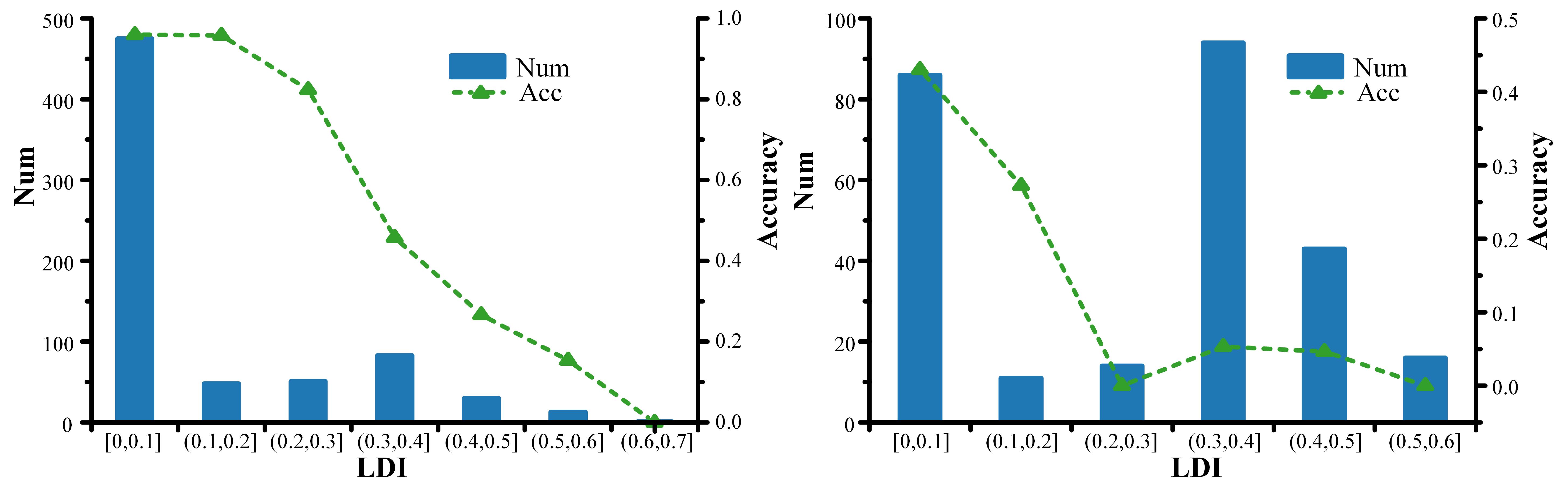}
\vspace{-0.3in}
\caption{Sample distributions and average accuracies for different $LDI$ intervals of the majority (left) and the minority class (right) of Citeseer. In the majority class, samples concentrate in small $LDI$ intervals. In the minority class, samples disperse more uniformly.}
\vspace{-0.15in}
\label{fig:fig4}
\end{figure}

\textbf{Average accuracies of different \textit{LDI} intervals.} Furthermore, we analyze the average accuracies of nodes in different $LDI$ intervals. As shown in Figure \ref{fig:fig5}, nodes concentrate in small $LDI$ intervals, and the classification accuracy decreases as the $LDI$ increases on different data sets. Nodes with large $LDI$s are more likely to exchange information with heterophily nodes in their neighborhoods, resulting in low accuracies. As shown in Figure \ref{fig:fig4}, within a specific category, the average accuracies of different $LDI$ intervals show that samples with large $LDI$ values in the majority classes can also be more likely to be misclassified. Thus establishing a reasonable node-level index to supplement the existing category-level index is necessary.

\begin{figure*}[ht]
    \centering
    \includegraphics[width=0.8\linewidth]{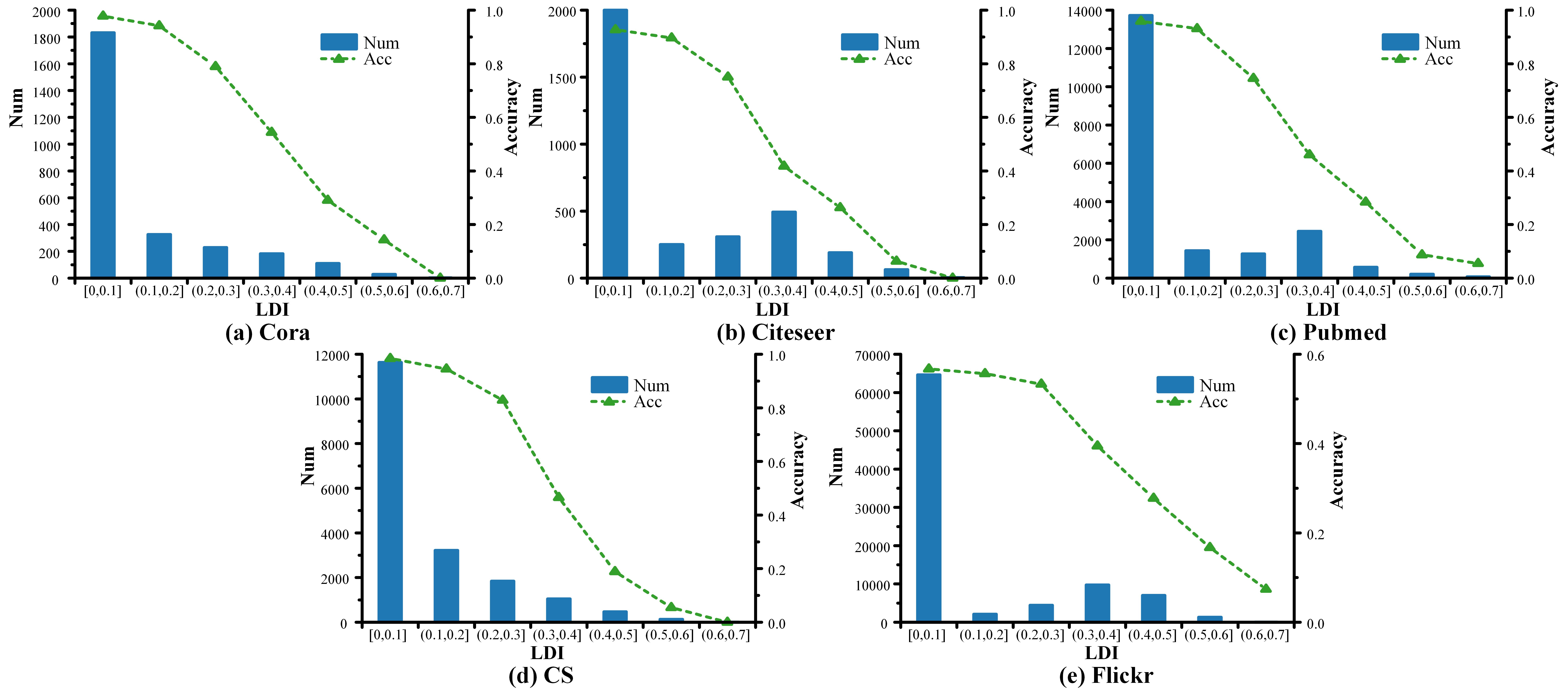}
    \vspace{-0.15in}
    \caption{For all data sets, nodes concentrate in small $LDI$ intervals. Nodes with large $LDI$s have low accuracies.}
    \label{fig:fig5}
    \vspace{-0.15in}
\end{figure*}

\textbf{Relationship between \textit{LDI} and misclassification under different $\textbf{\#layers}$.} The performance of GNNs decreases with the increase of layers (e.g., $\#layers>4$), which is mainly caused by oversmoothing. We explore the relationship between $LDI$ and the misclassification caused by oversmoothing. Let the correctly predicted sample set be $R_n$ when the number of layers equals to $n$. We calculate the ratio of the average $LDI$ of the newly wrongly predicted samples (i.e., $\overline{R_n}\cap R_{n-1}$) when the layers increase by one to the average $LDI$ of the correctly predicted samples before the layer increases, that is,
\begin{equation}\label{1}
r_n=\frac{LDI_{avg}(\overline{R_n}\cap R_{n-1})}{LDI_{avg}(R_{n-1})},
\end{equation}
where $LDI_{avg}(\Omega)$ represents the average $LDI$ of set $\Omega$.

In Figure~\ref{fig:fig6}, as the $\#layers$ increases, the value of $r_n$ is greater than 1 in nearly all cases. It shows that as the number of layers increases, samples with large $LDI$s are more likely to be misclassified. More results are shown in the appendix.

\begin{figure}[]
    \centering
    \includegraphics[width=1\linewidth]{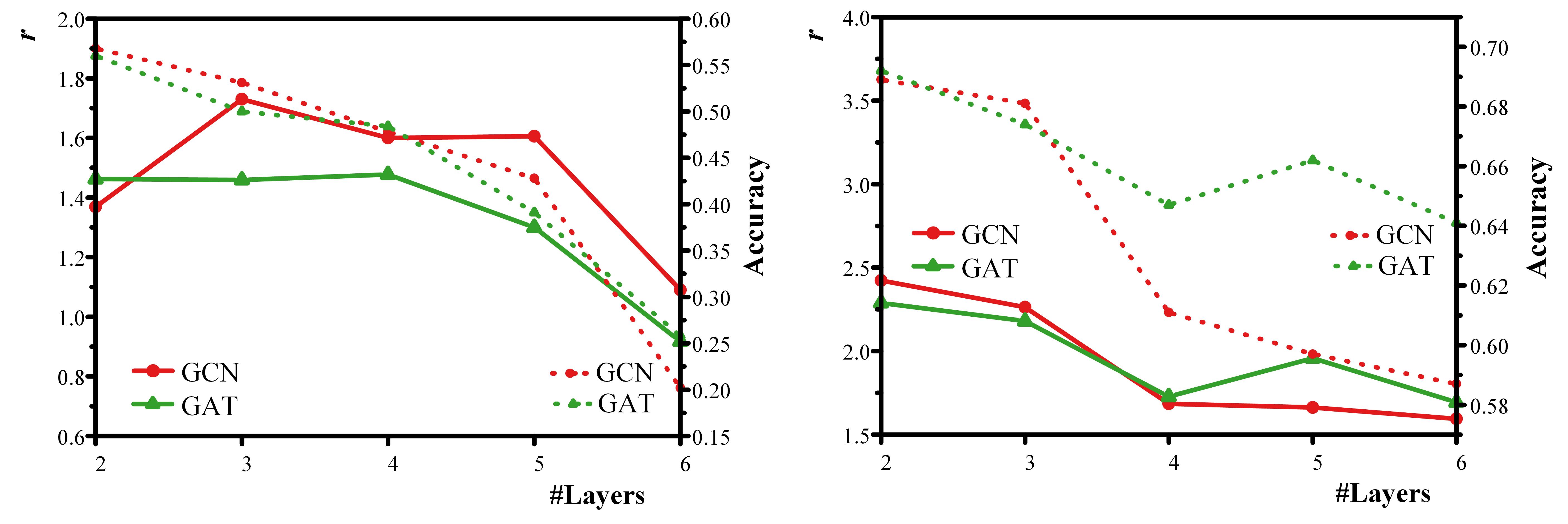}
    \vspace{-0.2in}
    \caption{Variations in $r$ (solid lines) and accuracy (dot lines) with increasing \#layers using GCN and GAT under insductive (left) and transductive (right) settings on the Citeseer data set.}
    \label{fig:fig6}
\vspace{-0.15in}
\end{figure}

To sum up, the category distribution of neighbor nodes that $LDI$ relies on is related closely to the global category distribution as shown in Figure~\ref{fig:fig3}. $LDI$ reflects the possibility of each node being misclassified as shown in Figure~\ref{fig:fig5}. Therefore, on the basis of the global distribution and $LDI$, we proposed several new methods to tackle the imbalance for GNNs in the succeeding section.

\section{Methodology}
A new loss and four new methods, namely, iFL, GRS, GRW, GML, and GBBN, are proposed in this section.

\vspace{-0.15in}
\subsection{Improved Focal Loss (iFL). }
Focal loss is designed to solve the imbalance problem in object classification. For binary classification, the focal loss can be defined as
\begin{equation}\label{FL}
\mathcal{L}_{FL} = -\sum_i\alpha_{y_i}(1-p_{i,y_i})^{\gamma}\text{log}(p_{i,y_i}),
\end{equation}
where $p_{i,y_i}$ is the probability that node $i$ is correctly classified; $\alpha_{y_i}$ and ${\gamma}$ are hyperparameters. For binary classification, the value of $\alpha_{y_i}$ is easy to set. Nevertheless, the number of categories for GNN benchmark data sets is ususally large, resulting in that it is difficult to set $\alpha_{y_i}$ properly. Therefore, to avoid the manually setting for $\alpha_{y_i}$, an improved focal loss is defined as follows:
\begin{equation}\label{fl_ours}
\mathcal{L}_{iFL} = -\sum_i(1-\overline{p}_{y_i})^{\gamma}\text{log}(p_{i,y_i}),
\end{equation}
where $\overline{p}_{y_i}$ is the average probability that nodes in the category $y_i$ are correctly classified. For categories with low classification accuracy, the value of $(1-\overline{p}_{y_i})^{\gamma}$ will be large.

\subsection{Graph Re-sampling (GRS) and Graph Re-weighting (GRW). } \label{4.1}

Re-sampling/re-weighting essentially resamples/reweights the the samples during network training to strengthen the learning of the minority classes.

Our proposed GRS and GRW are based on re-sampling and re-weighting, respectively. In this study, re-sampling and re-weighting are performed according to the numbers of samples in different categories, and give more weights to the minority classes with smaller sample sizes. $LDI$ is also utilized. Suppose the category of $x_i$ is $c$. $N_{c}$ is the number of samples of category $c$. $N$ is the total number of samples of the data set. For $x_i$, there are three weighting strategies:
\begin{description}
\item[W1:] Label-based weighting:
    \begin{equation}
    \vspace{-0.1in}
    w^{L}_{i}=\frac{N}{N_{c}}.
    \end{equation}
\item[W2:] $LDI$-based weighting\footnote{Because $LDI$s of some samples are zero, to prevent the weights from being zero, $e^{LDI_i}$ is used as the weight.}:
    \begin{equation}
    \vspace{-0.05in}
    w^{D}_{i}=e^{LDI_i}.
    \vspace{-0.05in}
    \end{equation}
\item[W3:] Combination of label and $LDI$:
    \begin{equation}
    \vspace{-0.05in}
    w^{LD}_{i}=w^{L}_{i}\cdot w^{D}_{i}.
    \vspace{-0.05in}
    \end{equation}
\end{description}

\subsection{Graph Metric Learning (GML). }

Metric learning aims to learn a feature space in which the distances between homophily samples are close and the distances between heterophily samples are far. Dong et al.~\cite{8353718_7} applied a triplet loss-based metric learning to deal with the imbalance problem in image classification.

Triplet loss is used in this study. The key to triplet loss-based metric learning is to select the appropriate triplets as training samples. When choosing a triplet, the anchor point is first selected; the positive and negative sample pairs corresponding to it are then selected from other samples. According to the probability of being easy to be classified correctly, anchor points are divided into easy ones and hard ones. Because the hard triplets are not conducive to the learning of the model, Wang et al.~\cite{Wang_2019_ICCV_8} proposed that only choosing the easy anchors would obtain better performance. 

In this study, $LDI$ is used to measure the hardness of a sample. The top 10\% samples with the highest LDI values are removed (i.e. the top 10\% hardest samples are removed). For the remaining samples, the larger the $LDI$ of a sample, the more likely it is to be selected. For each anchor, in its $K$-neighborhood, the negative samples are selected according to the closeness of distances. Outside its $K$-neighborhood, the positive samples are selected according to the further distances. The candidates are shown in Figure~\ref{fig:fig7}. Let the t$th$ triplet be $(x_t,x_{t+},x_{t-})$, its score can be calculated by
\begin{equation}
score_t=max(0,m+d(x_t,x_{t+})-d(x_t,x_{t-})),
\end{equation}
where $m$ refers to margin, and $d(\cdot)$ represents the distance between two feature vectors. Assume that $T$ is the set of selected triplets with relatively high scores. The triplet loss can be represented as
\begin{equation}
\vspace{-0.05in}
\mathcal{L}_{ML}=\frac{1}{|T|}\sum_{t} score_t.
\end{equation}
The triplet loss and the focal loss form the final loss function. Following the setting of Zhou et al.~\cite{Zhou_2020_CVPR_9}, we also adopt a cumulative learning strategy to combine the above two losses: 
\begin{equation}
\mathcal{L}_{GML}=f(e)\mathcal{L}_{ML}+\mathcal{L}_{iFL},
\end{equation}
where $f(e)$ refers to the weight in the current epoch.

\begin{figure}[h]
    \centering
    \includegraphics[width=0.45\linewidth]{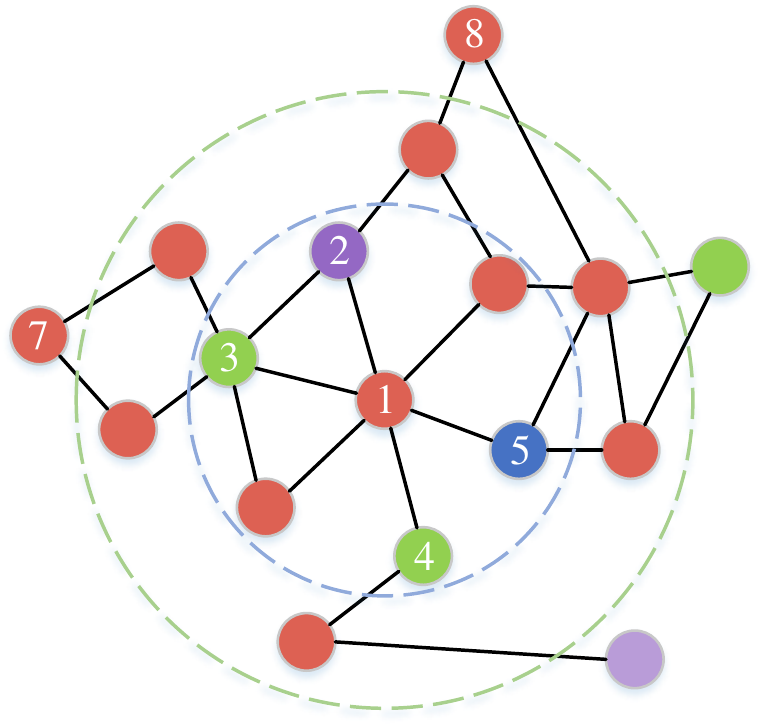}
    \vspace{-0.1in}
    \caption{Node 1 is the selected anchor. Let $K=2$, nodes 2, 3, 4, and 5 represent negative candidate nodes, and the outermost two red nodes 7 and 8 are positive candidate nodes.}
    \label{fig:fig7}
\vspace{-0.2in}
\end{figure}

\subsection{Graph Bilateral-Branch Network (GBBN).}

Our GBBN is based on the bilateral-branch network proposed by Zhou et al.~\cite{Zhou_2020_CVPR_9} to solve the imbalance in visual recognition tasks. The network consists of two branches: the conventional learning branch and the re-balancing branch. The conventional branch samples uniformly from the original data, maintaining the original data distribution for feature learning. The re-balancing branch is a reverse sampling, which aims to increase the probability of minority classes and reduce data imbalance~\cite{DBLP:journals/corr/abs-1709-01450_10}.

Figure~\ref{fig:fig8} shows the structure of GBBN. The conventional (upper) branch uses conventional models such as GCN and GAT. The re-balancing (lower) branch no longer adopts the same structure as the upper branch, but uses DropEdge GNN~\cite{rong2020dropedge_14} to reduce information exchange between nodes. In the experiments of this study, we drop all edges to reduce the training complexity. In addition, the proposed iFL is used rather than the standard cross entropy loss. Comparing with GNN-XML proposed in \cite{zong2020gnn} which aims to overcome the imbalance for text classification, the settings of two branches and the loss functions are different. 

\begin{figure*}[h]
    \centering
    \includegraphics[width=0.75\linewidth]{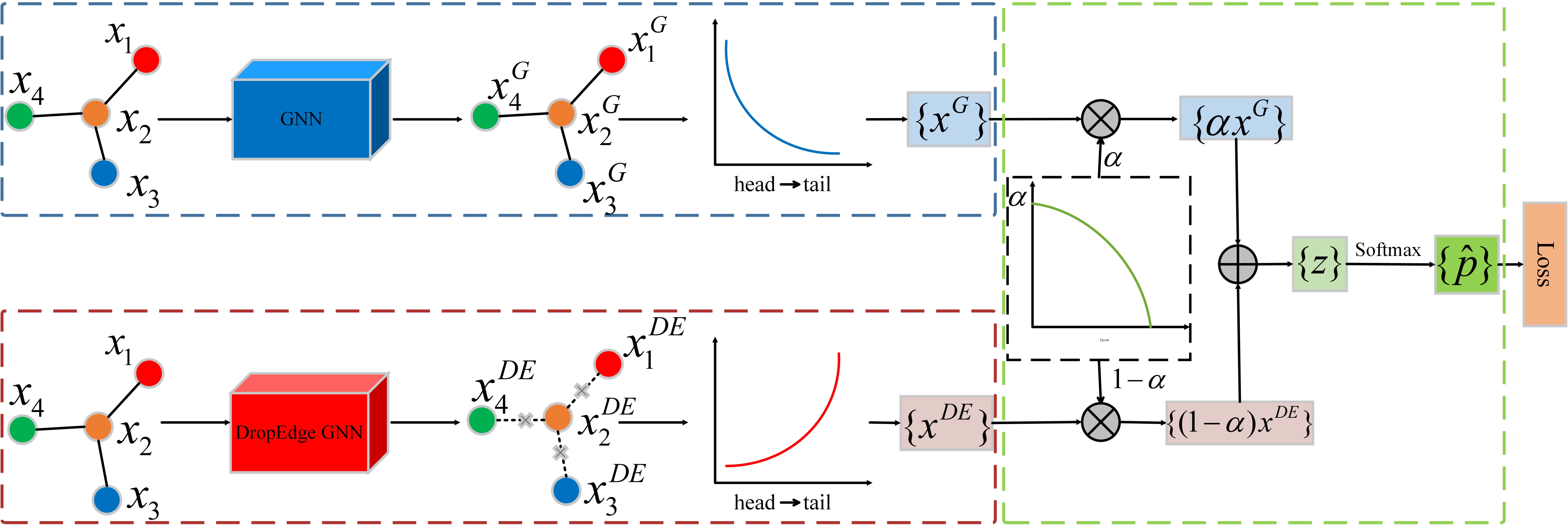}
    \vspace{-0.1in}
    \caption{The network structure of GBBN.}
    \label{fig:fig8}
\vspace{-0.1in}
\end{figure*}

The upper branch learns the representations of the nodes according to uniform sampling. The sampled nodes also present an imbalanced distribution. The reverse sampling of the lower branch takes the label and $LDI$ of a node into account. For samples with higher $LDI$s, the sampling probabilities are greater, which not only balances the data from the category-level, but also focuses on the samples that are easily misclassified from the node-level. For this branch, this study leverages the three weighting strategies defined in Section~\ref{4.1}, namely, W1, W2, and W3.

The graph is entered into the two branches. Two logit outputs of each sample under GNN and DropEdge GNN are obtained and denoted as $x^G$ and $x^{DE}$, respectively. Then, a weighted logit vector is obtained by controlling the adaptive weight parameter $\alpha$. The final output logit vector can be obtained by using the following formula: 
\begin{equation}
z=\alpha x^G+(1-\alpha)x^{DE},
\end{equation}
where $z\in \mathbb{R}^C$ refers to the final logic vector. The prediction is $\hat{p}=\text{Softmax}(z)$.

The classification loss of GBBN can be expressed as follows:
\begin{equation}
\mathcal{L}_{GBBN}=\alpha \mathcal{L}_{iFL}(\hat{p},y^G)+(1-\alpha)\mathcal{L}_{iFL}(\hat{p},y^{DE}).
\end{equation}

The network training also adopts the cumulative learning strategy. By changing the parameter $\alpha$, the output logit and the loss also vary, resulting in that different training stages have different learning emphases. In the early stage, the emphasis is on the upper branch. Then, the emphasis gradually shots to the minority categories and the lower branch. This is reasonable. With the increase of training epoch, minority nodes (especially nodes with large $LDI$s) should rely on their own representations more than neighborhoods because their neighborhoods contain more heterophily nodes than others. In our experiments, we use linear decay, i.e. $\alpha =1-e/Epoch$, where $e$ is the current epoch and $Epoch$ is the expected max epoch.

\section{Experiment}

We design extensive experiments to evaluate the effectiveness of the four proposed methods. Both transductive and inductive settings are tested.

\subsection{Settings. }
\textbf{Data sets.} Five\footnote{The rest three data sets Reddit~\cite{zeng2020graphsaint_24}, Amazon-Computer~\cite{zeng2020graphsaint_24}, and Amazon-Photo~\cite{zeng2020graphsaint_24} shown in the appendix are relatively large. A subgraph sampling-based strategy (e.g., GraphSAINT~\cite{zeng2020graphsaint_24}) should be employed. We leave it as our future work.} benchmark data sets in Figure \ref{fig:fig3} are used: Cora~\cite{shi2020multi}, Citeseer, Pubmed~\cite{zhao2020pairnorm_15}, CoauthorCS (CS)~\cite{zhao2020pairnorm_15}, and Flickr~\cite{zeng2020graphsaint_24}. For Cora and Citeseer, the train, validation, and test division ratio is 2:4:4; for others, the division ratio is 0.5:4.5:5. Previous transductive studies on the above data sets do not shuffle training data in their experiments. Data sets with and without shuffle are considered in the experiments. When data shuffle is adopted, the average results of five randomly shuffle are recorded. The data statistics are summarized in the appendix. All five graphs are class-imbalanced as shown in Figure~\ref{fig:fig3}. 

\textbf{Hyperparameters.} For GCN, SGC, and GAT, the \#hidden units of Cora, Citeseer, and Pubmed, dropout rate, and $l2$ regularization penalty settings are the same as~\cite{zhao2020pairnorm_15}. For Flickr, the \#hidden unit is set the same as~\cite{zeng2020graphsaint_24}. The max epoch is set to 1500. For GML, $K$ is set to 2, the margin is set to 0.2, $\epsilon$ is set to 0.001, and $\rho$ is set to 0.3. For GBBN, $\alpha$ is set to 0.5 in the validation and testing phase. In the comparison with DG-GCN, the settings in \cite{shi2020multi} are followed. 

\subsection{Results. }\label{results}

\textbf{Comparison with DR-GCN.} DR-GCN \cite{shi2020multi} only focuses on the transductive setting. Therefore, only the weighting strategy W1 is used in GRS, GRW, and the lower branch of GBBN. The reason why W2 and W3 are not used is because the input data under transductive setting is the entire graph, when calculating the $LDI$s of the training set, the labels of the validation set and the test set will be referred to, resulting in label leak in the validation and test sets. The CS data set is not involved because this data set is not used in the DR-GCN study. The results in Table \ref{tab:tab2} show that our methods are better than DR-GCN on Cora, Citeseer, and Pubmed data sets. The accuracies of DR-GCN are directly from \cite{shi2020multi}.

\begin{table}[]
\caption{Accuracy of the competing methods.}
\vspace{-0.15in}
\label{tab:tab2}
\begin{tabular}{l|lll}
\bottomrule
Data set                  & Cora               & Citeseer       & Pubmed          \\\hline
DR-GCN~\cite{shi2020multi}& 0.741              & 0.677          & 0.817           \\\hline
GRS(GCN, W1)              & 0.745              & 0.685          & 0.820           \\
GRW(GCN, W1)              & \textbf{0.780}     & 0.657          & 0.817           \\
GBBN(GCN, W1)             & 0.756              & \textbf{0.703} & \textbf{0.827}  \\\bottomrule
\end{tabular}
\vspace{-0.2in}
\end{table}

\textbf{Results on the transductive setting with shuffle.}  For the same reason for why W2 and W3 are not used in the transductive setting, GML is also not utilized in this experiment. The accuracy, G-mean, and macro F1-score are used as the evaluation metrics. The results are shown in Tables~\ref{tab:tab3} and \ref{tab:tab3.1}. More results are shown in the appendix.

\begin{table*}[]
\centering
\caption{Accuracy and G-mean on the transductive setting with shuffle.}
\vspace{-0.1in}
\label{tab:tab3}
\begin{tabular}{l|ll|ll|ll|ll}
\bottomrule
Data set  & \multicolumn{2}{c|}{Cora}       & \multicolumn{2}{c|}{Citeseer}   & \multicolumn{2}{c|}{Pubmed}     & \multicolumn{2}{c}{CS}      \\ \hline
         & Acc            & G-mean         & Acc            & G-mean         & Acc            & G-mean         & Acc            & G-mean         \\ \hline
GCN       & 0.339          & 0.449          & 0.354          & 0.523          & 0.580           & 0.658          & 0.370           & 0.456          \\
GRS(GCN,W1)  & 0.343          & 0.420           & 0.357          & 0.527          & 0.584          & 0.661          & 0.366          & 0.450           \\ 
GRW(GCN,W1)  & 0.347          & 0.483          & 0.393          & 0.567          & 0.581          & 0.667          & 0.199          & 0.382          \\
GBBN(GCN,W1) & 0.586          & 0.710           & 0.581          & 0.714          & 0.783          & 0.833          & 0.798          & 0.811          \\ \hline
SGC       & 0.327          & 0.405          & 0.350           & 0.519          & 0.571          & 0.638          & 0.238          & 0.265          \\
GRS(SGC,W1)  & 0.331          & 0.415          & 0.351          & 0.520           & 0.571          & 0.638          & 0.239          & 0.266          \\
GRW(SGC,W1)  & 0.149          & 0.367          & 0.202          & 0.373          & 0.549          & 0.680           & 0.027          & 0.307          \\
GBBN(SGC,W1) & \textbf{0.669} & \textbf{0.747} & \textbf{0.676} & \textbf{0.764} & \textbf{0.838} & \textbf{0.876} & \textbf{0.888} & \textbf{0.889} \\ \hline
GAT       & 0.337          & 0.422          & 0.360           & 0.530           & 0.560           & 0.618          & 0.254          & 0.335          \\
GRS(GAT,W1)  & 0.338          & 0.414          & 0.352          & 0.522          & 0.568          & 0.623          & 0.250           & 0.320           \\
GRW(GAT,W1)  & 0.307          & 0.472          & 0.375          & 0.558          & 0.580           & 0.666          & 0.206          & 0.370           \\
GBBN(GAT,W1) & 0.553          & 0.663          & 0.575          & 0.693          & 0.804          & 0.845          & 0.782          & 0.804  \\ \bottomrule
\end{tabular}
\end{table*}

The following observations can be obtained: (1) Overall, GRS, GRW, and GBBN performed better as compared to the original models. (2) GBBN performed better than GRS and GRW overall. (3) The effectiveness of these three methods indicates that our label-based weighting strategy (W1) is effective.
\begin{table*}[h]
\begin{minipage}{\columnwidth}
\renewcommand{\arraystretch}{1.3}
\caption{Macro F1-score on the transductive setting with shuffle.}
\vspace{-0.1in}
\label{tab:tab3.1}
\centering
\setlength{\tabcolsep}{0.6mm}{
\begin{tabular}{l|llll}
\bottomrule
Data set   & Cora           & Citeseer       & Pubmed         & CS             \\ \hline
GCN       & 0.223          & 0.307          & 0.561          & 0.227          \\
GRS(GCN,W1)  & 0.175          & 0.314          & 0.566          & 0.223          \\
GRW(GCN,W1)  & 0.262          & 0.366          & 0.572          & 0.068          \\
GBBN(GCN,W1) & 0.554          & 0.549          & 0.783          & 0.689          \\ \hline
SGC       & 0.161          & 0.303          & 0.536          & 0.041          \\
GRS(SGC,W1)  & 0.176          & 0.304          & 0.536          & 0.041          \\
GRW(SGC,W1)  & 0.065          & 0.056          & 0.547          & 0.006          \\
GBBN(SGC,W1) & \textbf{0.623} & \textbf{0.626} & \textbf{0.838} & \textbf{0.826} \\ \hline
GAT       & 0.180           & 0.311          & 0.499          & 0.118          \\
GRS(GAT,W1)  & 0.168          & 0.307          & 0.502          & 0.106          \\
GRW(GAT,W1)  & 0.249          & 0.350           & 0.570           & 0.112          \\
GBBN(GAT,W1) & 0.492          & 0.519          & 0.803          & 0.680         \\ \bottomrule
\end{tabular}}
\end{minipage}
\begin{minipage}{\columnwidth}
\renewcommand{\arraystretch}{1.3}
\caption{Macro F1-score on the inductive setting based on SGC with shuffle.}
\vspace{-0.1in}
\label{tab:tab4.1}
\centering
\setlength{\tabcolsep}{0.6mm}{
\begin{tabular}{l|lllll}
\bottomrule
Data set  & Cora           & Citeseer       & Pubmed         & CS             & Flickr         \\ \hline
SGC      & 0.329          & 0.439          & 0.606          & 0.164          & 0.124          \\
GML      & 0.330           & 0.436          & 0.607          & 0.166          & 0.124          \\ \hline
GRS(W1)  & 0.439          & 0.495          & 0.631          & 0.327          & 0.140           \\
GRS(W2)  & 0.321          & 0.478          & 0.608          & 0.154          & 0.126          \\
GRS(W3)  & 0.447          & 0.498          & 0.633          & 0.358          & 0.138 \\ \hline
GRW(W1)  & 0.458          & 0.508          & 0.640           & 0.403          & 0.118          \\
GRW(W2)  & 0.183          & 0.202          & 0.499          & 0.072          & 0.091          \\
GRW(W3)  & 0.455          & 0.459          & 0.642          & 0.382          & 0.103          \\ \hline
GBBN(W1) & 0.637          & 0.639          & 0.837          & 0.849          & 0.140           \\
GBBN(W2) & 0.631          & 0.626          & 0.836          & 0.811          & 0.142          \\
GBBN(W3) & \textbf{0.662} & \textbf{0.644} & \textbf{0.838} & \textbf{0.853} & \textbf{0.197}        \\ \bottomrule
\end{tabular}}
\end{minipage}
\vspace{-0.1in}
\end{table*}

\textbf{Results on the inductive setting with shuffle.} In this part, we compare the effect of the model considering whether the $LDI$ is used in the inductive setting. Accordingly, all the weighting strategies (W1, W2, and W3) are leveraged. The macro F1-score and accuracy/G-mean based on SGC with shuffle are shown in Tables~\ref{tab:tab4.1} and \ref{tab:tab4}, respectively. When without shuffle, the results are shown in the appendix. The results on the inductive setting based on GCN and GAT are shown in the appendix.

\begin{table*}[]
\caption{Accuracy and G-mean on the inductive setting based on SGC with shuffle.}
\vspace{-0.15in}
\label{tab:tab4}
\begin{tabular}{l|ll|ll|ll|ll|ll}
\bottomrule
Data set  & \multicolumn{2}{c|}{Cora}       & \multicolumn{2}{c|}{Citeseer}   & \multicolumn{2}{c|}{Pubmed}     & \multicolumn{2}{c|}{CS}         & \multicolumn{2}{c}{Flickr}      \\ \hline
          & Acc           & G-mean         & Acc            & G-mean         & Acc            & G-mean         & Acc            & G-mean         & Acc            & G-mean         \\ \hline
SGC      & 0.450           & 0.529          & 0.504          & 0.633          & 0.642          & 0.692          & 0.346         & 0.386          & 0.440           & 0.372          \\
GML      & 0.449          & 0.530           & 0.503          & 0.633          & 0.643          & 0.693          & 0.347         & 0.387          & 0.440           & 0.372          \\ \hline
GRS(W1)  & 0.482          & 0.643          & 0.529          & 0.670           & 0.648          & 0.710           & 0.469        & 0.530          & 0.338          & 0.385          \\
GRS(W2)  & 0.454          & 0.528          & 0.527          & 0.658          & 0.643          & 0.694          & 0.334          & 0.375          & 0.440           & 0.373          \\
GRS(W3)  & 0.502          & 0.635          & 0.527          & 0.670           & 0.650           & 0.711          & 0.490        & 0.568          & 0.329          & \textbf{0.417}          \\ \hline
GRW(W1)  & 0.518          & 0.641          & 0.542          & 0.679          & 0.652          & 0.716          & 0.508          & 0.598          & 0.353          & 0.369          \\
GRW(W2)  & 0.367          & 0.448          & 0.332          & 0.488          & 0.598          & 0.637          & 0.291          & 0.336          & 0.424          & 0.353          \\
GRW(W3)  & 0.493          & 0.653          & 0.468          & 0.645          & 0.652          & 0.719          & 0.462          & 0.611          & 0.178          & 0.400            \\ \hline
GBBN(W1) & 0.666          & 0.762          & 0.675          & 0.770           & 0.835          & 0.878          & 0.894         & 0.905          & 0.448          & 0.384          \\
GBBN(W2) & 0.675          & 0.757          & 0.669          & 0.761          & 0.835          & 0.876          & 0.882          & 0.881          & 0.441          & 0.380      \\
GBBN(W3) & \textbf{0.687} & \textbf{0.784} & \textbf{0.683} & \textbf{0.774} & \textbf{0.837} & \textbf{0.879} & \textbf{0.896} & \textbf{0.908} & \textbf{0.456} & 0.389    \\ \bottomrule
\end{tabular}
\vspace{-0.15in}
\end{table*}

The following observations can be obtained: (1) The overall performances of our three methods (i.e., GRS, GRW, and GBBN) are better than the conventional model SGC. (2) When with shuffle, GBBN achieved the best performance. GML did not improve the performance of its base models. The reason is discussed in the next subsection. (3) Among the three weighting strategies, the combination of label and $LDI$ weighting strategy (i.e., W3) performed best overall, indicating that the index $LDI$ does contain useful cues for training. 

\vspace{-0.1in}
\subsection{Discussion. }

The better performances of our three methods (GML is poor) over existing methods indicate the importance of considering the distribution characteristics of the imbalance. The combination of label and $LDI$ weighting strategy W3 is better in most cases under the inductive setting, indicating that $LDI$ contains meaningful training cues. Our future research will study the significance of this index in terms of larger size of the neighborhood (only 1-neighborhood is considered in this study) and more layers.

Further, we discuss why GML failed in the experiment. The principle of metric learning is to decrease the distance between samples of the same category and increase the distance between samples of different categories. However, in the graph data, edge connections exist among samples. The network in training exchanges information among the feature representations of connected samples during training, which will cause metric learning to reduce the distance between an anchor and a positive sample. However, it is also likely to reduce the distance between the anchor and the heterophily samples in the positive sample's neighbors. The feature coupling caused by data connection causes the metric learning strategy proposed in this study to be unsuitable for graph data sets. As far as we know, there is currently no research on metric learning specifically for GNNs, and exploring more sophisticated methods in the future would be a worthwhile step.

\section{Conclusion}

This study focuses on the imbalanced distribution in learning with GNNs for graph data. Most benchmark data sets used in GNN studies exhibit an imbalanced distribution over categories. A node-level index called $LDI$ is defined to establish the connection between the imbalance and misclassification. Initial qualitative and quantitative analyses between $LDI$ and GNNs performances are conducted to reveal that samples with high $LDI$ values are more likely to be misclassified when layers increase. A new loss and four new methods (i.e. iFL, GRS, GRW, GML, and GBBN) are proposed based on $LDI$. Comparative experiment results show that the proposed methods are better than DR-GCN. Overall, the proposed GBBN with iFL and $LDI$ performs better than the other methods when with shuffle. The effectiveness of W3 indicates that the index $LDI$ does contain useful cues for training. 


\bibliographystyle{ACM-Reference-Format}


\setcounter{table}{0}
\setcounter{figure}{0}
\renewcommand{\thetable}{S\arabic{table}}
\renewcommand{\thefigure}{S\arabic{figure}}

\section*{Appendix}

\appendix

\section{Category distributions on three different data sets.}

As shown in Figure \ref{fig:figS1}, these three GNN benchmark data sets show imbalanced or highly-skewed distributions.

\begin{figure}[h]
    \centering
    \includegraphics[width=1.02\linewidth]{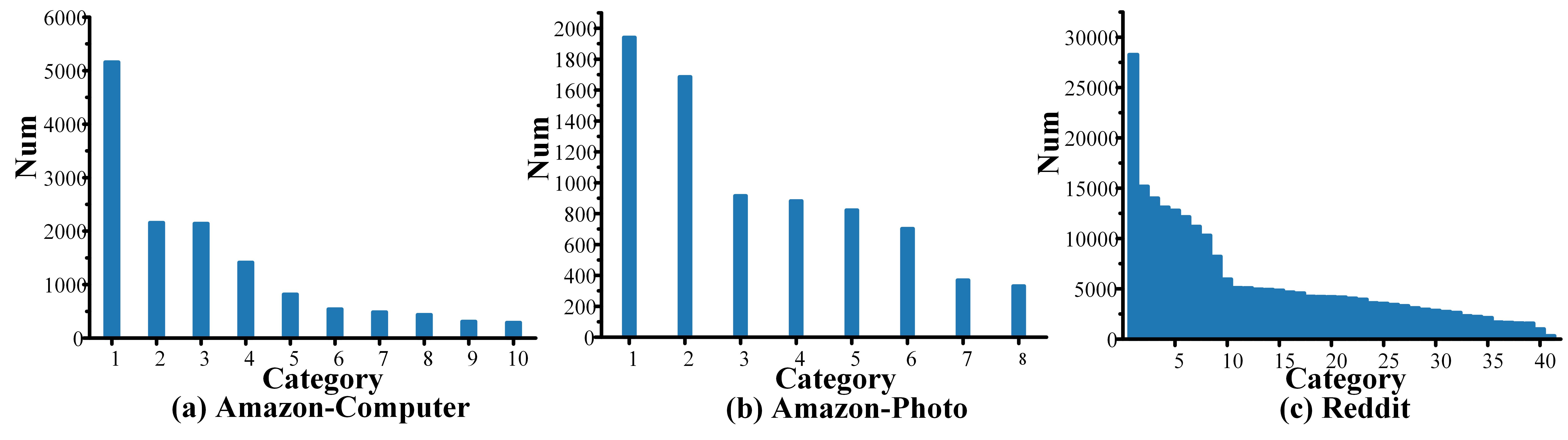}
     \vspace{-0.2in}
    \caption{Category distributions on three different datasets.}
    \label{fig:figS1}
    \vspace{-0.2in}
\end{figure}

\section{Proof of proposition.}

\begin{proposition}
Given a node $x_i$ belonging to the c$th$ category. When the proportion of homophily nodes in its neighborhood is fixed (i.e., $p_{i,{c}}$ is fixed), the less categories of heterophily nodes, the larger the $LDI$.
\end{proposition}

\begin{proof}
Let $\sum_{c'\ne c}p_{i,c'}=1-p_{i,c}=\Delta$, then
\begin{equation}
    LDI_i = \frac{1}{\sqrt{2}}\sqrt{(1-p_{i,c})^2+\sum_{c'\ne c}p_{i,c'}^2}.
\end{equation}

By using Cauchy Inequality, we have
\begin{equation}
    \frac{(\sum_{c'\ne c}p_{i,c'})^2}{C-1} \leq \sum_{c'\ne c}p_{i,c'}^2 \leq (\sum_{c'\ne c}p_{i,c'})^2.
\end{equation}

If and only if $\forall c'$, $p_{i,c'}=\frac{\Delta}{C-1}$, then
\begin{equation}
    \frac{(\sum_{c'\ne c}p_{i,c'})^2}{C-1} = \sum_{c'\ne c}p_{i,c'}^2 = \frac{\Delta ^2}{(C-1)^2}.
\end{equation}

If and only if $p_{i,c''}=1-p_{i,c}$ and $p_{i,c}=0 (c'\ne c'')$, then
\begin{equation}
    \sum_{c' \ne c}p_{i,c'}^2 = (\sum_{c' \ne c}p_{i,c'})^2 = \Delta ^2.
\end{equation}

Accordingly, 
\begin{equation}
    \Delta \sqrt{\frac{1+(C-1)^2}{2(C-1)^2}} \leq LDI_i \leq \Delta.
\end{equation}

The upper bound is attained only when all the heterophily nodes belong to the same category. To sum up, the more concentrated the categories of heterophily nodes, the larger the $LDI$. The proof ends.
\end{proof}

\section{More validations in $r$ and accuracy with increasing \#layers.}

In Figure~\ref{fig:figS2}, as the $\#layers$ increase, the value of $r_n$ is greater than 1 in almost all cases. It shows that as the number of layers increases, samples with higher $LDI$s are more likely to be misclassified.

\begin{figure}[h]
    \centering
    \includegraphics[width=\linewidth]{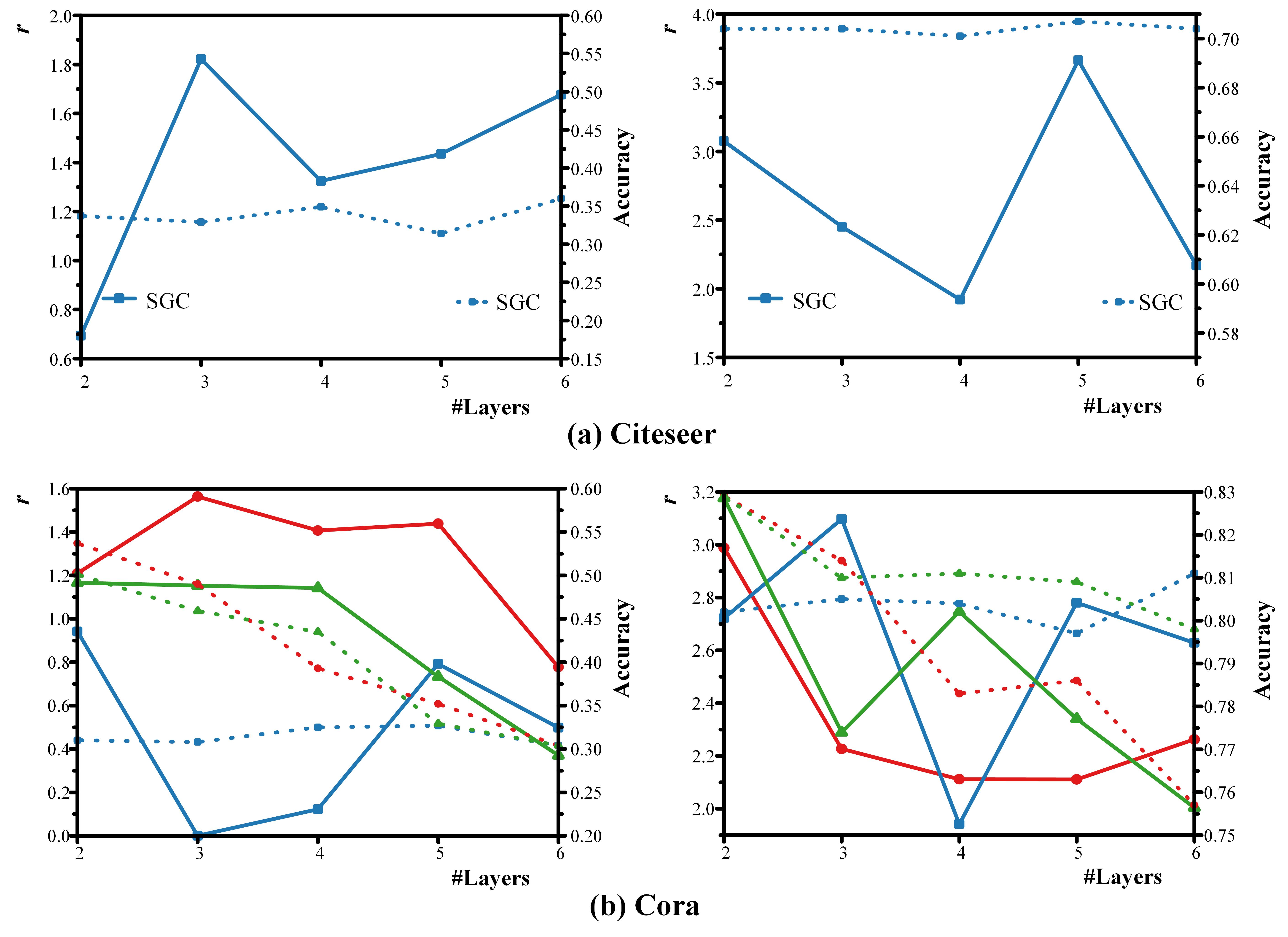}
    \vspace{-0.3in}
    \caption{(a) Variations in $r$ (solid lines) and accuracy (dot lines) with increasing \#layers using SGC under insductive (left) and transductive (right) settings on the Citeseer data set. (b) Variations in $r$ (solid lines) and accuracy (dot lines) with increasing \#layers using GCN, GAT and SGC under insductive (left) and transductive (right) settings on the Cora data set.}
    \label{fig:figS2}
\end{figure}

\section{Dataset statistics.}

The data statistics are shown in Table \ref{tab:tabS1}.

\begin{table}[h]
\vspace{-0.1in}
\centering
\caption{Data statistics of the five Graph datasets.}
\vspace{-0.15in}
\label{tab:tabS1}
\begin{tabular}{ccccc}
\bottomrule
Name       & \#Node & \#Edge & \#Features & \#Class \\ \hline
Cora       & 2,708   & 5,429   & 1,433       & 7       \\ 
Citeseer   & 3,327   & 4,732   & 3,703       & 6       \\ 
Pubmed     & 19,717  & 44,338  & 500        & 3       \\ 
CoauthorCS & 18,333  & 81,894  & 6,805       & 15      \\ 
Flickr     & 89,250  & 899,756 & 500        & 7     \\\bottomrule
\end{tabular}
\vspace{-0.2in}
\end{table}

\section{Results on the transductive setting with shuffle.}

In this part, the results on transductive setting without shuffle are shown in Tables \ref{tab:tabS2} and \ref{tab:tabS3}, respectively. The following observations can be obtained: (1) Overall, GRS, GRW, and GBBN performed better as compared to the original models. (2) Our label-based  sampling strategy (W1) is effective. (3) GRW performed better than GRS and GBBN overall.

\begin{table*}[]
\centering
\caption{Accuracy and G-mean on transductive setting without shuffle.}
\vspace{-0.1in}
\label{tab:tabS2}
\begin{tabular}{l|ll|ll|ll|ll}
\bottomrule
Data set  & \multicolumn{2}{c|}{Cora}       & \multicolumn{2}{c|}{Citeseer}   & \multicolumn{2}{c|}{Pubmed}     & \multicolumn{2}{c}{CS}      \\ \hline
         & Acc            & G-mean         & Acc            & G-mean         & Acc            & G-mean         & Acc            & G-mean         \\ \hline
GCN          & 0.849          & 0.903          & 0.748          & 0.814          & 0.841          & 0.878          & 0.918         & 0.938         \\ 
GRS(GCN,W1)  & 0.854          & 0.905          & 0.744          & 0.812          & 0.840          & 0.878          & 0.917         & 0.936         \\ 
GRW(GCN,W1)  & 0.845          & 0.901          & 0.748          & 0.819          & 0.828          & 0.871          & 0.882         & 0.928         \\
GBBN(GCN,W1) & 0.839          & 0.898          & 0.737          & 0.812          & 0.831          & 0.891          & 0.929         & 0.939         \\ \hline
SGC          & 0.847          & 0.894          & 0.736          & 0.802          & 0.816          & 0.858          & 0.881         & 0.880          \\
GRS(SGC,W1)  & 0.848          & 0.894          & 0.739          & 0.804          & 0.816          & 0.858          & 0.904         & 0.924         \\
GRW(SGC,W1)  & 0.713          & 0.864          & 0.706          & 0.795          & 0.775          & 0.834          & 0.111         & 0.522         \\
GBBN(SGC,W1) & 0.796          & 0.860          & 0.736          & 0.802          & 0.842          & 0.884          & 0.906         & 0.914         \\ \hline
GAT          & 0.855          & 0.904          & 0.748          & 0.810          & 0.846          & 0.881          & 0.914         & 0.934         \\
GRS(GAT,W1)  & 0.849          & 0.905          & 0.743          & 0.810          & 0.845          & 0.880          & 0.914         & 0.936         \\
GRW(GAT,W1)  & \textbf{0.857} & \textbf{0.907} & \textbf{0.754} & \textbf{0.824} & \textbf{0.858} & \textbf{0.892} & 0.894         & \textbf{0.940}         \\ 
GBBN(GAT,W1) & 0.835          & 0.890          & 0.735          & 0.804          & 0.855          & 0.875          & \textbf{0.930}& 0.934 \\ \bottomrule
\end{tabular}
\end{table*}

\begin{table*}[h]
\begin{minipage}{\columnwidth}
\renewcommand{\arraystretch}{1.3}
\caption{Macro F1-score on transductive setting without shuffle.}
\vspace{-0.1in}
\label{tab:tabS3}
\centering
\setlength{\tabcolsep}{0.6mm}{
\begin{tabular}{l|llll}
\bottomrule
Dataset   & Cora           & Citeseer       & Pubmed         & CS             \\ \hline
GCN          & 0.836          & 0.696          & 0.835          & 0.895          \\
GRS(GCN,W1)  & 0.841          & 0.692          & 0.835          & 0.893          \\
GRW(GCN,W1)  & 0.833          & 0.706          & 0.824          & 0.858          \\
GBBN(GCN,W1) & 0.827          & 0.694          & 0.827          & 0.901          \\ \hline
SGC          & 0.834          & 0.672          & 0.811          & 0.813          \\
GRS(SGC,W1)  & 0.835          & 0.676          & 0.811          & 0.874          \\
GRW(SGC,W1)  & 0.724          & 0.673          & 0.774          & 0.158          \\
GBBN(SGC,W1) & 0.771          & 0.673          & 0.841          & 0.867          \\ \hline
GAT          & 0.843          & 0.683          & 0.840          & 0.889          \\
GRS(GAT,W1)  & 0.837          & 0.685          & 0.839          & 0.892          \\
GRW(GAT,W1)  & \textbf{0.845} & \textbf{0.714} & \textbf{0.854} & 0.871          \\
GBBN(GAT,W1) & 0.820          & 0.678          & 0.852          & \textbf{0.902}         \\ \bottomrule
\end{tabular}}
\end{minipage}
\begin{minipage}{\columnwidth}
\renewcommand{\arraystretch}{1.3}
\caption{Macro F1-score on inductive setting based on SGC without shuffle.}
\vspace{-0.1in}
\label{tab:tabS4}
\centering
\setlength{\tabcolsep}{0.6mm}{
\begin{tabular}{l|lllll}
\bottomrule
Dataset  & Cora           & Citeseer       & Pubmed         & CS             & Flickr         \\ \hline
SGC      & 0.601          & 0.634          & 0.808          & 0.717          & 0.148          \\
GML      & 0.595          & 0.632          & 0.809          & 0.718          & 0.148          \\ \hline
GRS(W1)  & 0.723          & 0.664          & 0.811          & 0.858          & 0.109          \\
GRS(W2)  & 0.596          & 0.640           & 0.808         & 0.701          & 0.157          \\
GRS(W3)  & 0.707          & 0.680           & 0.811         & 0.855          & \textbf{0.212} \\ \hline
GRW(W1)  & 0.744          & 0.680           & 0.811         & 0.872          & 0.201          \\
GRW(W2)  & 0.166          & 0.443          & 0.625          & 0.310          & 0.098          \\
GRW(W3)  & \textbf{0.748} & \textbf{0.682} & \textbf{0.837} & \textbf{0.878} & 0.042          \\ \hline
GBBN(W1) & 0.690          & 0.637          & 0.836          & 0.868          & 0.178          \\
GBBN(W2) & 0.673          & 0.634          & 0.836          & 0.854          & 0.151          \\
GBBN(W3) & 0.691          & 0.621          & 0.811          & 0.874          & 0.179        \\ \bottomrule
\end{tabular}}
\end{minipage}
\end{table*}


\section{Results on the inductive setting.}

\subsection{Results on the inductive setting based on SGC without shuffle.}

In this part, the results on inductive setting based on SGC without shuffle are shown in Tables \ref{tab:tabS4} and \ref{tab:tabS5}, respectively. The following observations can be obtained: (1) The overall performances of our three methods (i.e., GRS, GRW, and GBBN) are better than the conventional model SGC. (2) When without shuffle, GRW achieved the best performance. (3) The effectiveness of W3 indicates that the index $LDI$ does contain useful cues for training.

\begin{table*}[t]
\caption{Accuracy and G-mean on inductive setting based on SGC without shuffle.}
\vspace{-0.1in}
\label{tab:tabS5}
\begin{tabular}{l|ll|ll|ll|ll|ll}
\bottomrule
Dataset  & \multicolumn{2}{c|}{Cora}       & \multicolumn{2}{c|}{Citeseer}   & \multicolumn{2}{c|}{Pubmed}     & \multicolumn{2}{c|}{CS}         & \multicolumn{2}{c}{Flickr}      \\ \hline
          & Acc           & G-mean         & Acc            & G-mean         & Acc            & G-mean         & Acc            & G-mean         & Acc            & G-mean         \\ \hline
SGC      & 0.659          & 0.720           & 0.701         & 0.777          & 0.812          & 0.852          & 0.824          & 0.818          & 0.456          & 0.392          \\ 
GML      & 0.653          & 0.715          & 0.700          & 0.776          & 0.812          & 0.852          & 0.824          & 0.819          & 0.456          & 0.391          \\\hline
GRS(W1)  & 0.750          & 0.830           & 0.711         & 0.791          & 0.813          & 0.858          & 0.889          & 0.910           & 0.417          & 0.361          \\
GRS(W2)  & 0.655          & 0.715          & 0.705          & 0.780           & 0.811         & 0.851          & 0.816          & 0.809          & \textbf{0.457} & 0.397          \\
GRS(W3)  & 0.732          & 0.822          & 0.712          & 0.799          & 0.813          & 0.858          & 0.888          & 0.908          & 0.275          & \textbf{0.457} \\ \hline
GRW(W1)  & 0.764          & 0.844          & 0.691          & 0.801          & 0.811          & 0.859          & 0.895          & 0.923          & 0.308          & 0.429          \\
GRW(W2)  & 0.366          & 0.415          & 0.582          & 0.687          & 0.699          & 0.725          & 0.616          & 0.570           & 0.433          & 0.355          \\
GRW(W3)  & \textbf{0.767} & \textbf{0.846} & \textbf{0.719} & \textbf{0.803} & \textbf{0.837} & \textbf{0.878} & \textbf{0.908} & \textbf{0.925}  & 0.075          & 0.362          \\ \hline
GBBN(W1) & 0.721          & 0.801          & 0.691          & 0.776          & 0.811          & 0.860          & 0.907          & 0.918          & 0.441          & 0.407          \\
GBBN(W2) & 0.703          & 0.781          & 0.691          & 0.773          & 0.836          & 0.875          & 0.900          & 0.904          & 0.453           & 0.401          \\
GBBN(W3) & 0.721          & 0.804          & 0.673          & 0.763          & 0.835          & 0.877          & 0.899           & 0.921          & 0.441          & 0.407   \\ \bottomrule
\end{tabular}
\end{table*}

\subsection{Results on the inductive setting based on GCN and GAT. }

We compare the effect of the model considering whether the $LDI$ is used on the inductive setting. Accordingly, all the three weighting strategies (W1, W2, and W3) are leveraged. The results based on GCN are shown in Tables \ref{tab:tabS6}, \ref{tab:tabS7}, \ref{tab:tabS8} and \ref{tab:tabS9}, and the results based on GAT are shown in Tables~\ref{tab:tabS10}, \ref{tab:tabS11}, \ref{tab:tabS12} and \ref{tab:tabS13}.

The following observations can be obtained: (1) The overall performances of our three methods (i.e., GRS, GRW, and GBBN) are better than the conventional models GCN and GAT. (2) When with shuffle, GBBN achieved the best performance based on GCN and GAT. (3) When without shuffle, GBBN achieved the best performance based on GCN, and GRW achieved the best performance based on GAT. (4) The effectiveness of W2 and W3 indicates that the index $LDI$ does contain useful cues for training. 

Overall, the proposed GBBN and GRW algorithms performed better than the other methods. The effectiveness of W2 and W3 indicates that the index $LDI$ does contain useful cues for training.

\begin{table*}[]
\caption{Accuracy and G-mean on inductive setting based on GCN with shuffle.}
\vspace{-0.1in}
\label{tab:tabS6}
\begin{tabular}{l|ll|ll|ll|ll|ll}
\bottomrule
Dataset  & \multicolumn{2}{c|}{Cora}       & \multicolumn{2}{c|}{Citeseer}   & \multicolumn{2}{c|}{Pubmed}     & \multicolumn{2}{c|}{CS}         & \multicolumn{2}{c}{Flickr}      \\ \hline
         & Acc            & G-mean         & Acc            & G-mean         & Acc            & G-mean         & Acc            & G-mean         & Acc            & G-mean         \\ \hline
GCN      & 0.524          & 0.627          & 0.554          & 0.678          & 0.666          & 0.719          & 0.471          & 0.528          & 0.442          & 0.376          \\
GML      & 0.499          & 0.626          & 0.534          & 0.668          & 0.665          & 0.721          & 0.442          & 0.513          & 0.442          & 0.377          \\ \hline
GRS(W1)  & 0.522          & 0.635          & 0.541          & 0.680           & 0.670           & 0.727          & 0.500            & 0.569       & 0.429          & 0.358          \\
GRS(W2)  & 0.509          & 0.614          & 0.538          & 0.670           & 0.670           & 0.725          & 0.481          & 0.536         & 0.430          & 0.377          \\
GRS(W3)  & 0.485          & 0.645          & 0.540           & 0.676          & 0.670           & 0.729          & 0.504          & 0.571         & 0.391          & 0.400            \\ \hline
GRW(W1)  & 0.526          & 0.648          & 0.555          & 0.692          & 0.664          & 0.729          & 0.469          & 0.555          & 0.411          & 0.308          \\
GRW(W2)  & 0.514          & 0.602          & 0.552          & 0.670           & 0.662          & 0.711          & 0.421          & 0.462          & 0.366          & \textbf{0.435} \\
GRW(W3)  & 0.525          & 0.648          & 0.545          & 0.691          & 0.670           & 0.733          & 0.415          & 0.568          & 0.394          & 0.381          \\ \hline
GBBN(W1) & 0.640           & 0.769          & 0.667          & 0.773         & 0.807            & 0.853          & 0.843          & 0.862          & 0.437          & 0.370           \\
GBBN(W2) & \textbf{0.674} & \textbf{0.778} & \textbf{0.668} & \textbf{0.775} & \textbf{0.808}  & \textbf{0.854} & \textbf{0.846} & \textbf{0.863} & \textbf{0.443}  & 0.358          \\
GBBN(W3) & 0.643          & 0.744          & 0.659          & 0.774          & 0.807          & 0.852          & 0.845           & 0.862          & 0.435           & 0.366          \\ \bottomrule
\end{tabular}
\end{table*}

\begin{table*}[h]
\begin{minipage}{\columnwidth}
\renewcommand{\arraystretch}{1.3}
\caption{Macro F1-score on inductive setting based on GCN with shuffle.}
\vspace{-0.1in}
\label{tab:tabS7}
\centering
\setlength{\tabcolsep}{0.6mm}{
\begin{tabular}{l|lllll}
\bottomrule
Dataset  & Cora           & Citeseer       & Pubmed         & CS             & Flickr         \\ \hline
GCN                          & 0.449          & 0.507          & 0.643          & 0.325          & 0.129          \\
GML                          & 0.440          & 0.496          & 0.645          & 0.293          & 0.131          \\ \hline
GRS(W1)                      & 0.450          & 0.513          & 0.652          & 0.367          & 0.101          \\
GRS(W2)                      & 0.432          & 0.499          & 0.651          & 0.334          & 0.131          \\
GRS(W3)                      & 0.446          & 0.508          & 0.655          & 0.373          & 0.181          \\ \hline
GRW(W1)                      & 0.463          & 0.526          & 0.653          & 0.346          & \textbf{0.190} \\
GRW(W2)                      & 0.408          & 0.486          & 0.632          & 0.229          & 0.116          \\
GRW(W3)                      & 0.465          & 0.521          & 0.658          & 0.362          & 0.163          \\ \hline
GBBN(W1)                     & 0.625          & 0.640          & 0.808          & 0.786          & 0.123          \\
GBBN(W2)                     & \textbf{0.644} & \textbf{0.642} & \textbf{0.809} & \textbf{0.801} & 0.101          \\
GBBN(W3)                     & 0.602          & 0.638          & 0.808          & 0.707          & 0.115           \\ \bottomrule
\end{tabular}}
\end{minipage}
\begin{minipage}{\columnwidth}
\renewcommand{\arraystretch}{1.3}
\caption{Macro F1-score on inductive setting based on GCN without shuffle.}
\vspace{-0.1in}
\label{tab:tabS8}
\centering
\setlength{\tabcolsep}{0.6mm}{
\begin{tabular}{l|lllll}
\bottomrule
Dataset  & Cora           & Citeseer       & Pubmed         & CS             & Flickr         \\ \hline
GCN      & 0.726          & 0.652          & 0.833          & 0.857          & 0.148          \\
GML      & 0.694          & 0.627          & 0.834          & 0.793          & 0.144          \\ \hline
GRS(W1)  & 0.718          & 0.686          & 0.837          & 0.871          & 0.197           \\
GRS(W2)  & 0.725          & 0.665          & 0.836          & 0.860          & 0.150           \\
GRS(W3)  & 0.715          & 0.678          & 0.835          & 0.871          & \textbf{0.204}          \\ \hline
GRW(W1)  & 0.753          & 0.665          & 0.832          & 0.861          & 0.173          \\
GRW(W2)  & 0.715          & 0.661          & 0.830          & 0.778          & 0.136          \\
GRW(W3)  & \textbf{0.759} & 0.678          & 0.832          & 0.856          & 0.149          \\ \hline
GBBN(W1) & 0.717          & 0.679          & 0.845          & 0.870          & 0.119          \\
GBBN(W2) & 0.710          & 0.660          & 0.845          & 0.843          & 0.182          \\
GBBN(W3) & 0.716          & \textbf{0.688} & \textbf{0.846} & \textbf{0.872} & 0.186          \\ \bottomrule
\end{tabular}}
\end{minipage}
\end{table*}

\begin{table*}[]
\caption{Accuracy and G-mean on inductive setting based on GCN without shuffle.}
\vspace{-0.1in}
\label{tab:tabS9}
\begin{tabular}{l|ll|ll|ll|ll|ll}
\bottomrule
Dataset  & \multicolumn{2}{c|}{Cora}       & \multicolumn{2}{c|}{Citeseer}   & \multicolumn{2}{c|}{Pubmed}     & \multicolumn{2}{c|}{CS}         & \multicolumn{2}{c}{Flickr}      \\ \hline
          & Acc           & G-mean         & Acc            & G-mean         & Acc            & G-mean         & Acc            & G-mean         & Acc            & G-mean         \\ \hline
GCN      & 0.751         & 0.820          & 0.702          & 0.783          & 0.838          & 0.871          & 0.895          & 0.907          & 0.456          & 0.394          \\
GML      & 0.735         & 0.801          & 0.676          & 0.767          & 0.838          & 0.872          & 0.859          & 0.873          & 0.453          & 0.391          \\ \hline
GRS(W1)  & 0.744         & 0.816          & 0.724          & 0.804          & 0.840          & 0.876          & 0.900          & 0.920          & 0.361          & 0.423          \\
GRS(W2)  & 0.753         & 0.817          & 0.710          & 0.791          & 0.840          & 0.876          & 0.894          & 0.909          & \textbf{0.457} & 0.396          \\ 
GRS(W3)  & 0.746         & 0.815          & 0.723          & 0.799          & 0.839          & 0.875          & 0.897          & 0.920          & 0.361          & \textbf{0.444}  \\\hline
GRW(W1)  & \textbf{0.780}& 0.840          & 0.721          & 0.795          & 0.834          & 0.873          & 0.889          & 0.915          & 0.219          & 0.412          \\
GRW(W2)  & 0.741         & 0.805          & 0.723          & 0.793          & 0.834          & 0.869          & 0.867          & 0.863          & 0.449          & 0.381          \\
GRW(W3)  & 0.744         & \textbf{0.846} & 0.721          & 0.799          & 0.835          & 0.875          & 0.887          & 0.910          & 0.168          & 0.441          \\ \hline
GBBN(W1) & 0.744         & 0.827          & \textbf{0.725} & 0.801          & 0.847          & 0.880          & 0.907          & 0.919          & 0.442          & 0.367          \\
GBBN(W2) & 0.737         & 0.816          & 0.715          & 0.793          & 0.846          & 0.880          & 0.897          & 0.901          & 0.431          & 0.410           \\
GBBN(W3) & 0.745         & 0.822          & 0.719          & \textbf{0.805} & \textbf{0.848} & \textbf{0.884} & \textbf{0.908} & \textbf{0.921} & 0.401          & 0.413   \\ \bottomrule
\end{tabular}
\end{table*}

\begin{table*}[]
\caption{Accuracy and G-mean on inductive setting based on GAT with shuffle.}
\vspace{-0.1in}
\label{tab:tabS10}
\begin{tabular}{l|ll|ll|ll|ll|ll}
\bottomrule
Dataset  & \multicolumn{2}{c|}{Cora}       & \multicolumn{2}{c|}{Citeseer}   & \multicolumn{2}{c|}{Pubmed}     & \multicolumn{2}{c|}{CS}         & \multicolumn{2}{c}{Flickr}      \\ \hline
         & Acc            & G-mean         & Acc            & G-mean         & Acc            & G-mean         & Acc            & G-mean         & Acc            & G-mean         \\ \hline
GAT      & 0.476          & 0.565          & 0.532          & 0.659          & 0.666          & 0.722          & 0.430           & 0.465          & 0.442          & 0.377          \\ 
GML      & 0.491          & 0.578          & 0.528          & 0.655          & 0.659          & 0.710          & 0.414          & 0.466          & 0.442          & 0.374          \\ \hline
GRS(W1)  & 0.488          & 0.629          & 0.538          & 0.678          & 0.669          & 0.740          & 0.522          & 0.636          & 0.384          & 0.376          \\
GRS(W2)  & 0.480          & 0.572          & 0.530          & 0.660          & 0.653          & 0.712          & 0.423          & 0.473          & 0.440           & 0.374          \\
GRS(W3)  & 0.484          & 0.627          & 0.509          & 0.655          & 0.673          & 0.742          & 0.498          & 0.613          & 0.358          & 0.400   \\ \hline
GRW(W1)  & 0.493          & 0.632          & 0.544          & 0.682          & 0.669          & 0.733          & 0.504          & 0.615          & 0.311          & 0.403          \\
GRW(W2)  & 0.481          & 0.567          & 0.547          & 0.671          & 0.661          & 0.712          & 0.387          & 0.432          & 0.439          & 0.373 \\
GRW(W3)  & 0.481          & 0.615          & 0.537          & 0.677          & 0.674          & 0.743          & 0.485          & 0.607          & 0.306          & \textbf{0.415} \\ \hline
GBBN(W1) & 0.628          & 0.727          & 0.657          & 0.750          & 0.818          & 0.860          & 0.882          & 0.884          & 0.431          & 0.363          \\
GBBN(W2) & 0.624          & 0.701          & 0.648          & 0.738          & 0.820          & 0.859          & 0.881          & 0.878          & 0.434          & 0.383          \\
GBBN(W3) & \textbf{0.629} & \textbf{0.728} & \textbf{0.658} & \textbf{0.757} & \textbf{0.822} & \textbf{0.870} & \textbf{0.884} & \textbf{0.888} & \textbf{0.449} & 0.370   \\ \bottomrule
\end{tabular}
\end{table*}

\begin{table*}[h]
\begin{minipage}{\columnwidth}
\renewcommand{\arraystretch}{1.3}
\caption{Macro F1-score on inductive setting based on GAT with shuffle.}
\vspace{-0.1in}
\label{tab:tabS11}
\centering
\setlength{\tabcolsep}{0.6mm}{
\begin{tabular}{l|lllll}
\bottomrule
Dataset  & Cora           & Citeseer       & Pubmed         & CS             & Flickr         \\ \hline
GCN      & 0.363          & 0.474          & 0.649          & 0.238          & 0.132          \\
GML      & 0.382          & 0.471          & 0.631          & 0.251          & 0.127          \\ \hline
GRS(W1)  & 0.435          & 0.508          & 0.664          & 0.433          & 0.145          \\
GRS(W2)  & 0.378          & 0.475          & 0.637          & 0.260          & 0.127          \\
GRS(W3)  & 0.432          & 0.474          & 0.667          & 0.412          & 0.183          \\ \hline
GRW(W1)  & 0.438          & 0.515          & 0.657          & 0.421          & 0.168          \\
GRW(W2)  & 0.374          & 0.489          & 0.634          & 0.213          & 0.126          \\
GRW(W3)  & 0.418          & 0.507          & 0.667          & 0.406          & \textbf{0.185}          \\ \hline
GBBN(W1) & 0.580          & 0.601          & 0.819          & 0.812          & 0.112          \\
GBBN(W2) & 0.554          & 0.588          & 0.820          & 0.803          & 0.137          \\
GBBN(W3) & \textbf{0.586} & \textbf{0.617} & \textbf{0.821} & \textbf{0.817} & 0.124         \\ \bottomrule
\end{tabular}}
\end{minipage}
\begin{minipage}{\columnwidth}
\renewcommand{\arraystretch}{1.3}
\caption{Macro F1-score on inductive setting based on GAT without shuffle.}
\vspace{-0.1in}
\label{tab:tabS12}
\centering
\setlength{\tabcolsep}{0.6mm}{
\begin{tabular}{l|lllll}
\bottomrule
Dataset  & Cora           & Citeseer       & Pubmed         & CS             & Flickr         \\ \hline
GCN      & 0.723          & 0.673          & 0.832          & 0.829          & 0.148          \\
GML      & 0.709          & 0.647          & 0.830           & 0.821          & 0.149          \\ \hline
GRS(W1)  & 0.747          & 0.684          & 0.835          & 0.852          & 0.151          \\
GRS(W2)  & 0.717          & 0.658          & 0.835          & 0.828          & 0.149          \\
GRS(W3)  & 0.740          & 0.681          & 0.835          & 0.857          & \textbf{0.191} \\ \hline
GRW(W1)  & 0.755          & 0.683          & 0.834          & 0.857          & 0.151          \\
GRW(W2)  & 0.728          & 0.665          & 0.836          & 0.789          & 0.139          \\
GRW(W3)  & \textbf{0.765} & \textbf{0.689} & \textbf{0.838} & \textbf{0.865} & 0.144          \\ \hline
GBBN(W1) & 0.691          & 0.623          & 0.836          & 0.858          & 0.143          \\
GBBN(W2) & 0.680          & 0.614          & 0.832          & 0.841          & 0.148          \\
GBBN(W3) & 0.695          & 0.628          & 0.832          & 0.849          & 0.142         \\ \bottomrule
\end{tabular}}
\end{minipage}
\end{table*}

\begin{table*}[t]
\caption{Accuracy and G-mean on inductive setting based on GAT without shuffle.}
\vspace{-0.1in}
\label{tab:tabS13}
\begin{tabular}{l|ll|ll|ll|ll|ll}
\bottomrule
Dataset  & \multicolumn{2}{c|}{Cora}       & \multicolumn{2}{c|}{Citeseer}   & \multicolumn{2}{c|}{Pubmed}     & \multicolumn{2}{c|}{CS}         & \multicolumn{2}{c}{Flickr}      \\ \hline
         & Acc            & G-mean         & Acc            & G-mean         & Acc            & G-mean         & Acc            & G-mean         & Acc            & G-mean         \\ \hline
GCN      & 0.742          & 0.811          & 0.716          & 0.796          & 0.838          & 0.868          & 0.889          & 0.891          & 0.456          & 0.394          \\
GML      & 0.732          & 0.801          & 0.713          & 0.788          & 0.836          & 0.866          & 0.882          & 0.888          & 0.456          & 0.394          \\ \hline
GRS(W1)  & 0.764          & 0.850          & 0.721          & 0.803          & 0.839          & 0.875          & 0.891          & 0.908          & 0.383          & 0.382          \\
GRS(W2)  & 0.744          & 0.810          & 0.702          & 0.786          & 0.839          & 0.871          & 0.888          & 0.891          & \textbf{0.457} & 0.397          \\
GRS(W3)  & 0.759          & 0.836          & 0.715          & 0.801          & 0.839          & 0.875          & 0.894          & 0.911          & 0.331          & 0.407          \\ \hline
GRW(W1)  & 0.780          & 0.849          & 0.718          & 0.803          & 0.836          & 0.873          & 0.888          & 0.911          & 0.202          & 0.442          \\
GRW(W2)  & 0.749          & 0.814          & 0.717          & 0.793          & 0.836          & 0.875          & 0.872          & 0.869          & 0.452          & 0.385          \\
GRW(W3)  & \textbf{0.788} & \textbf{0.856} & \textbf{0.725} & \textbf{0.806} & \textbf{0.840} & \textbf{0.876} & \textbf{0.907} & \textbf{0.919} & 0.181          & \textbf{0.444} \\ \hline
GBBN(W1) & 0.715          & 0.806          & 0.682          & 0.769          & 0.837          & 0.875          & 0.893          & 0.910           & 0.436          & 0.392          \\
GBBN(W2) & 0.711          & 0.791          & 0.672          & 0.761          & 0.839          & 0.874          & 0.898          & 0.902          & 0.451          & 0.395          \\
GBBN(W3) & 0.724          & 0.807          & 0.684          & 0.770          & 0.833          & 0.872          & 0.898          & 0.904          & 0.453          & 0.388   \\ \bottomrule
\end{tabular}
\end{table*}

\end{document}